\documentclass{article}
\usepackage{kr}

\usepackage{times}
\usepackage{soul}
\usepackage{url}
\usepackage[hidelinks]{hyperref}
\usepackage[utf8]{inputenc}
\usepackage[small]{caption}
\usepackage{graphicx}
\usepackage{amsmath}
\usepackage{amsthm}
\usepackage{thmtools}
\usepackage{booktabs}
\usepackage{algorithm}
\usepackage{algorithmic}
\usepackage{array,ragged2e}
\usepackage{scalerel}
\usepackage[capitalise,noabbrev]{cleveref}
\usepackage{enumitem}
\usepackage[all]{foreign}
\usepackage{tikz}

\usepackage{packages/symbols}
\usepackage{packages/pictures}

\newcommand{\myDots}{\ifmmode\mathinner{\ldotp\kern-0.1em\ldotp\kern-0.1em\ldotp}\else.\kern-0.1em.\kern-0.3em.\fi}

\usepackage{color}
\usepackage{amsthm}
\usepackage{amssymb}
\usepackage{bussproofs}
\usepackage{xspace}
\usepackage{upgreek}
\usepackage{multicol}

\newcommand\mydots{\hbox to 0.8em{.\hss.\hss.}}

\newtheorem{theorem}{Theorem}

\newtheorem{lemma}[theorem]{Lemma}

\theoremstyle{definition}
\newtheorem{definition}[theorem]{Definition}

\theoremstyle{remark}

\newcommand{\fig}{Fig.}

\newcommand{\sect}{Sect.}

\newcommand{\iffi}{\textit{iff} }
\def\phi{\varphi}

\definecolor{tim}{RGB}{0, 0, 250}
\definecolor{nik}{RGB}{0, 120, 0}

\newcommand{\np}{\mathrm{NP}}
\newcommand{\pspace}{\mathrm{PSPACE}}

\newcommand{\sufo}[1]{\mathrm{sufo}(#1)}

\newcommand{\clo}[1]{\mathcal{C}(#1)}

\newcommand{\spbox}[1]{\Box\nolimits_{#1}}
\newcommand{\spdia}[1]{\Diamond\nolimits_{#1}}
\newcommand{\sharper}{\preceq}

\newcommand{\sltlv}{\sltl^\vocab}
\newcommand{\sltl}{\mathrm{SLTL}}
\renewcommand{\sl}{\mathrm{SL}}

\newcommand{\vocab}{\mathcal{V}}
\newcommand{\pset}{\mathcal{P}}
\newcommand{\spset}{\mathcal{S}}
\newcommand{\dst}{\ast}
\renewcommand{\next}{\mathsf{X}}
\newcommand{\even}{\mathsf{F}}
\newcommand{\alw}{\mathsf{G}}
\renewcommand{\until}{\mathcal{U}}
\newcommand{\rel}{\mathcal{R}}

\newcommand{\langsltl}{\mathcal{L}_{\vocab}}
\newcommand{\ltl}{\mathrm{LTL}}

\newcommand\citet[1]{\citeauthor{#1}~(\citeyear{#1})}
\renewcommand\O{\mathcal{O}}

\newcommand{\stence}{\varepsilon}
\newcommand{\stenc}[1]{\mathrm{enc}(#1)}
\newcommand{\stencF}[1]{\mathrm{enc}(#1)}

\newcommand{\node}{\Gamma}
\newcommand{\nodei}[1]{\Gamma(#1)}

\newcommand{\csatomf}[1]{\mathrm{a}(#1)}

\newcommand{\cs}{c}
\newcommand{\constraintset}{\langle \stence, \ell, \Delta\rangle}
\newcommand{\cslong}[1]{\langle #1 \rangle}

\newcommand{\csenc}[1]{\stence(#1)}
\newcommand{\cslab}[1]{\ell(#1)}
\newcommand{\csfor}[1]{\Delta(#1)}
\renewcommand{\false}{\textbf{f}}

\newcommand{\setindex}{(\stence)}

\newcommand{\seti}[2]{(#2)}
\newcommand{\setid}[1]{\langle#1\rangle}
\newcommand{\setib}[1]{[#1]}

\renewcommand{\state}{\sigma}
\newcommand{\stseq}{\bar{\sigma}}

\newcommand{\premd}{\mathcal{D}}

\renewcommand{\branch}[1]{\overline{#1}}

\newcommand{\nest}[1]{\{#1\}}
\newcommand{\branchdef}{\branch{u} = \langle u_{0}, \ldots, u_{n} \rangle}
\newcommand{\branchvdef}{\branch{v} = \langle v_{0}, \ldots, v_{m} \rangle}
\newcommand{\mlabel}[1]{\mathcal{L}(#1)}
\newcommand{\mlabelf}{\mathcal{L}}

\newcommand{\cl}[1]{\mathrm{cl}(#1)}
\newcommand{\tstamp}{\mathcal{T}}
\newcommand{\tsec}{\bar\tstamp}
\newcommand{\run}{r}
\newcommand{\runs}{\mathbf{R}}
\newcommand{\csatom}{\cs^{\mathrm{a}}}

\newcommand{\ru}{\mathsf{Rule}}
\newcommand{\disr}{\mathsf{DIS}}
\newcommand{\untr}{\mathsf{UNT}}
\newcommand{\relr}{\mathsf{REL}}
\newcommand{\evenr}{\mathsf{EVE}}
\newcommand{\conr}{\mathsf{CON}}
\newcommand{\alwr}{\mathsf{ALW}}
\newcommand{\stepr}{\mathsf{STEP}}
\newcommand{\contrar}{\mathsf{CONTRADICTION}}
\newcommand{\loopr}{\mathsf{LOOP}}
\newcommand{\pruner}{\mathsf{PRUNE}}
\newcommand{\emptyr}{\mathsf{EMPTY}}
\newcommand{\sboxri}{\mathsf{BOX}^{1}}
\newcommand{\sboxrii}{\mathsf{BOX}^{2}}
\newcommand{\sdiar}{\mathsf{DIA}}

\newcommand{\spform}[1]{\mathsf{#1}}

\newcommand{\mathcom}[3]{ \newcommand{#1}[#2]{\mbox{$#3$}}}
\newcommand{\remathcom}[3]{ \renewcommand{#1}[#2]{\mbox{$#3$}}}
\mathcom{\imp}{0}{\ \rightarrow\ }            
\mathcom{\rimp}{0}{\ \leftarrow\ }            

\mathcom{\con}{0}{\ \wedge\ }                 
\mathcom{\dis}{0}{\ \vee\ }                   
\mathcom{\n}{0}{\neg}                     
\mathcom{\dimp}{0}{\ \leftrightarrow\ }       
\mathcom{\corresponds}{0}{\ \Lleftarrow\! \! \Rrightarrow\ }
\mathcom{\A}{0}{\forall}                  
\remathcom{\E}{0}{\exists}     

\def\Box{\mathop\square}
\def\Diamond{\mathop\lozenge}

\remathcom{\tuple}{1}{\langle #1 \rangle}

\def\sp{\hbox{$\spform{s'}$}\xspace}
\def\st{\hbox{$\spform{s}$}\xspace}

\def\standb#1{\Box\nolimits_{\spform{#1}}}
\def\standd#1{\Diamond\nolimits_{\spform{#1}}}

\def\standbe{\standb{e}}
\def\standbs{\standb{s}}

\def\standde{\standd{e}}
\def\standds{\standd{s}}

\def\model{\mathfrak{M}}

\xspace
\def\f\xspacestandtopre{\hbox{$\sigma\,$}\xspace}
\def\fpretov\xspacealue{\hbox{$\delta\,$}\xspace}

\def\ModSat#1||-#2{#1\models #2}
\def\NotModSat#1||-#2{#1\nvDash #2}

\newcommand{\skipit}[1]{} 
\newcommand{\addit}[1]{} 

\renewcommand{\N}{\mathbb{N}}
\renewcommand{\set}[1]{\left\{#1\right\}}

\renewcommand{\land}{\mathrel{\wedge}}
\renewcommand{\lor}{\mathrel{\vee}}

\renewcommand{\S}{\mathcal{S}}

\def\Stands{\mathcal{S}}

\def\E{\mathcal{E}}
\def\StandExps{\E_{\Stands}}
\def\ste{\spform{e}} 

\title{Standpoint Linear Temporal Logic}

\author{%
Nicola Gigante$^1$\and
Lucía Gómez Álvarez$^2$\and
Tim S. Lyon$^2$\\
\affiliations
$^1$Free University of Bozen-Bolzano, Italy\\
$^2$TU Dresden, Germany\\
\emails
nicola.gigante@unibz.it,
\{lucia.gomez\_alvarez,timothy\_stephen.lyon\}@tu-dresden.de
}

\begin{document}

\maketitle

\begin{abstract}
Many complex scenarios require the coordination of agents possessing unique points of view and distinct semantic commitments. In response, \emph{standpoint logic} ($\sl$) was introduced in the context of knowledge integration, allowing one to reason with diverse and potentially conflicting viewpoints by means of indexed modalities. Another multi-modal logic of import is \emph{linear temporal logic} ($\ltl$)---a formalism used to express temporal properties of systems and processes, having prominence in formal methods and fields related to artificial intelligence. In this paper, we present \emph{standpoint linear temporal logic} ($\sltl$), a new logic that combines the temporal features of $\ltl$ with the multi-perspective modelling capacity of $\sl$. We define the logic $\sltl$, its syntax, its semantics, establish its decidability and complexity, and provide a terminating tableau calculus to automate $\sltl$ reasoning. Conveniently, this offers a clear path to extend existing $\ltl$ reasoners with practical reasoning support for temporal reasoning in multi-perspective settings.
\end{abstract}

\section{Introduction}
\label{sec:introduction}


Reasoning about systems involving multiple agents is a core problem in
artificial intelligence, with countless applications studied over the last few
decades. A variety of formalisms have been introduced to model and reason about multi-agent scenarios; e.g. STIT (See To It That) logic~\cite{BelPer88,BelPerXu01}, BDI (Belief-Desire-Intention) logic~\cite{RaoGeo98,Bra87}, deontic agency logic~\cite{BerLyo21,Mur05}, and epistemic logic~\cite{DitHoeKoo07,Pla07}. Despite the usefulness of such logics in modelling the internal states of an agent, including the agent's attitudes, beliefs, and knowledge, such logics tend to be difficult to handle computationally~\cite{BalHerTro08,Lut06,RaoGeo98}.

In contrast to these approaches, \emph{standpoint logic}
($\sl$) has been recently introduced~\cite{gomez2021standpoint} as a relatively low-cost multi-agent logic with applications in the context of knowledge integration. Within the framework of standpoint logic, propositions may be `wrapped' within modalities of the form $\spbox{\st}$ and $\spdia{\st}$ with $\st$ a standpoint, allowing for declarations of the form $\spbox{\st} \phi$ (`according to $\st$, it is \emph{unequivocal} that $\phi$') and $\spdia{\st} \phi$ (`according to $\st$, it is \emph{conceivable} that $\phi$'). Such modalities capture the semantic commitments occurring at a particular standpoint and do not require the \emph{nesting} of semantic commitments within semantic commitments, which allows for $\sl$ to recover favourable computational properties; indeed, the satisfiability problem for $\sl$ is $\np$-complete~\cite{gomez2021standpoint}.


 A natural application of standpoint-based frameworks arises in the context of distributed and multi-agent systems, since they support the establishment of different, possibly conflicting specifications and their coordination. And, as temporal considerations often arise when modelling the behaviour of a system, it appears worthwhile to enhance standpoint logic with temporal operators, thus allowing for mutable states-of-affairs and changing standpoints to be explicitly described. To endow standpoint logic with the capacity to express dynamic concepts, we use the preferred formalism for modelling and expressing temporal notions, i.e., \emph{linear temporal logic} or $\ltl$~\cite{Pnueli77}. 
 
 $\ltl$ is a propositional modal logic interpreted over discrete, infinite sequences of
states. In the nearly five decades since its inception, $\ltl$ has gained
popularity 
as a specification language for
systems, and 
has found many applications in AI. For instance, $\ltl$ has been applied in \emph{automated planning}, temporally extended
goals~\cite{BacchusK98}, temporal planning~\cite{FoxL03,MayerLOP07},
timeline-based planning~\cite{DellaMonicaGMSS17}, planning over (partially
observable) Markov decision processes~\cite{BrafmanD19,BrafmanDP18},
reinforcement learning~\cite{DeGiacomoFIP20,HammondA0W21}, temporal description
logics~\cite{ArtaleKRZ14}, and temporal epistemic logics~\cite{vanBenthemGHP09}.
Computationally, temporal logics are usually quite hard; e.g. $\ltl$
satisfiability is $\pspace$-complete. However, many efficient techniques and tools
exists to deal with $\ltl$ specifications (\eg,~\citet{GeattiGM19};
\citet{LiYP0H14}; \citet{cavada2014nuxmv}), allowing for the logic to be routinely used in practice and industry.

In this paper, we \emph{fuse} the multi-perspective capabilities of $\sl$ with the
temporal features of $\ltl$, resulting in \emph{standpoint linear temporal
logic} ($\sltl$). $\sltl$ inherits the features of both $\sl$ and $\ltl$, letting us model both the
evolution of a system as well as changing standpoints over time. 
The result is a flexible formalism that
maintains the favourable computational properties of its components. In particular, we provide:
\begin{enumerate}
  \item A detailed syntax and semantics for $\sltl$;
    \item A \emph{tree-shaped} tableau calculus to facilitate and automate $\sltl$-reasoning;
  \item An analysis of the computational complexity of $\sltl$, which is found to be $\pspace$-complete (i.e., $\sltl$ is no harder than $\ltl$).
\end{enumerate}

Our tableau calculus uses \emph{quasi-model} based methods \cite{Wolter98SatisfiabilityDescriptionLogicsModalOperators}. It is inspired by the nested sequent calculi for propositional standpoint logic~\cite{LyoGom22} and built atop the tree-shaped tableau calculus for $\ltl$ provided by
Reynolds~\cite{Reynolds16a,GeattiGMR21}, which has many interesting features. In
particular, the tree shape of Reynolds tableau calculus allows it to be easily and efficiently traversed
\emph{symbolically}, \ie by means of Boolean formulas solved by off-the-shelf
SAT solvers, as done by the BLACK satisfiability checker~\cite{GeattiGM21}. By
maintaining the tree shape of the tableau and its overall structure, which is
extended to support standpoints, we pave the way for the adoption of symbolic
techniques for efficient reasoning in $\sltl$.

The paper is structured as follows: In \cref{sec:preliminaries}, we introduce the syntax and semantics
of $\sltl$, as well as exemplify how the logic may be applied. In \cref{sec:tableaux}, we define our tableau
calculus for $\sltl$, and describe the tableau-based algorithm that decides $\sltl$ formulae. Subsequently, in \cref{sec:proofs} we prove the calculus sound and complete, and in \cref{sec:complexity}, we prove that the satisfiability problem for $\sltl$ is $\pspace$-complete. 
 \cref{sec:conclusions} concludes and discusses future work.

\section{Standpoint Linear Temporal Logic}\label{sec:prelims} 
\label{sec:preliminaries}

 We formally introduce \emph{standpoint linear temporal logic} ($\sltl$), which fuses together propositional standpoint logic ($\sl$)~\cite{gomez2021standpoint} and linear temporal logic ($\ltl$)~\cite{Pnueli77}. We begin by explaining the various logical operators and modalities included in the language and demonstrate their applicability by means of an example. Subsequently, we provide a semantics for $\sltl$, defining the models used (\emph{temporal standpoint structures}) and clarifying how formulae are interpreted.

\subsection{Language}

 The logic $\sltl$ is built atop classical propositional logic, and therefore, employs propositional variables along with the connectives for negation $\neg$, disjunction $\lor$, and conjunction $\con$. In addition, our logic incorporates the temporal modalities from LTL; in particular, (1) the unary modalities $\next$, $\even$, and $\alw$, and (2) the binary modalities $\until$ and $\rel$. These modalities are read as follows: $\next \phi$ states `at the next moment $\phi$ holds', $\even \phi$ states `eventually $\phi$ holds', $\alw \phi$ states `always $\phi$ holds', $\phi \until \psi$ states `$\phi$ holds until $\psi$ holds', and $\phi \rel \psi$ is interpreted as the dual of $\phi \until \psi$. The formal semantics of these formulae can be found in \cref{def:semantic-clauses} below.
 
  We also employ the standpoint modalities $\spbox{\st}$ and $\spdia{\st}$, where $\st$ is taken from a finite set $\spset$ of standpoints, $\spbox{\st} \phi$ is read as `according to $\st$, it is unequivocal that $\phi$' and $\spdia{\st} \phi$ is read as `according to $\st$, it is conceivable that $\phi$'. Furthermore, we include formulae of the form $\st \sharper \sp$ indicating that the standpoint $\st$ is \emph{sharper} than $\sp$, i.e. $\st$ complies with $\sp$.


\begin{definition}[Formulae]\label{def:logical-languages} Let $\vocab = \langle \pset, \spset \rangle$ be a \emph{vocabulary}, where $\pset$ is a non-empty set of propositional variables and $\spset$ is a set of standpoint symbols containing the distinguished symbol $\dst$, called the \emph{universal standpoint}. We define the language $\langsltl$ to be the collection of all standpoint expressions of the form $\st \preceq \sp$ where $\st,\sp \in \spset$, and of all formulae $\phi$ generated via the following grammar in BNF:
$$
\phi ::= p \ | \ \neg p \ | \ (\phi \circ \phi) \ | \ \triangledown \phi 
$$
 where $\circ \in \{\lor, \land, \until, \rel\}$, $\triangledown \in \{
\next, \even, \alw\} \cup \{\spdia{\st}, \spbox{\st} \ | \ \st\in\spset\}$ and $p \in \pset$. We use $p$, $q$, $r$ $\ldots$ (potentially annotated) to denote propositional variables and $\phi$, $\psi$, $\chi$, $\ldots$ (potentially annotated) to denote formulae from $\langsltl$.\footnote{We have opted to employ formulae for standpoint LTL in \emph{negation normal form} as it will simplify the presentation of our tableaux later on.} We define the formulae $\neg \phi$ with $\phi$ a complex formula and $\phi \rightarrow \psi$ as usual.
\end{definition}

In order to calculate complexity bounds for our tableaux, it will be helpful to define the size of formulae.

\begin{definition}[Subformula, Size]\label{DEF:subformulas} We define the set of \emph{subformulae} of $\phi$, denoted $\sufo{\phi}$, recursively as follows:
\begin{itemize}

\item $\sufo{p} := \{p\}$ and $\sufo{\neg p} := \{\neg p\}$;

\item $\sufo{\triangledown \psi} := \{\triangledown \psi\} \cup \sufo{\psi}$;

\item $\sufo{\psi \circ \chi} := \{\psi \circ \chi\} \cup \sufo{\psi} \cup \sufo{\chi}$.

\end{itemize}
with $\circ \in \{\lor, \land, \until, \rel\}$, $\triangledown \in \{
\next, \even, \alw\} \cup \{\spdia{\st}, \spbox{\st} | \ \ste \in \StandExps\}$, and $p \in \pset$. We say that $\psi$ is a \emph{subformula} of $\phi$ \iffi $\psi \in \sufo{\phi}$. We define the \emph{size} of a formula $\phi$ in $\langsltl$, denoted $|\phi|$, accordingly: $|\phi| := |\sufo{\phi}|$. Last, we define the \emph{size} of a set of formulae $\Phi$, denoted $|\Phi|$, as: $|\Phi| := \Upsigma_{\phi\in\Phi} |\sufo{\phi}|$.
\end{definition}
 
\subsubsection{Example.} Medical devices require testing and certification prior to marketing and use by medical professionals. Albeit, regulations differ from country to country, giving rise to potentially conflicting standards and safety qualifications. For instance, Germany ($\mathsf{DE}$) and Italy ($\mathsf{IT}$) may agree that a medical device $\mathtt{X}$ has been deemed \emph{safe according to testing} ($\mathtt{X\_TestSafe}$), so long as it has been found that it never \emph{malfunctions} ($\mathtt{Malf}$). This judgement can be expressed by the $\sltl$ formula shown below, where $\dst$ is the universal standpoint which encodes that the formula is unequivocal from all perspectives (i.e. from the perspective of both $\mathsf{DE}$ and $\mathsf{IT}$).
\begin{equation}
\spbox{\dst} (\mathtt{X\_TestSafe} \rightarrow \alw \neg \mathtt{Malf})
\end{equation}
 Yet, despite the agreement on what makes $\mathtt{X}$ safe according to testing, each country may differ in how it considers $\mathtt{X}$ to be \emph{safe overall} ($\mathtt{X\_Safe}$).  It could be that Italy deems a device safe so long as it has been deemed safe according to testing or has been found \emph{safe by comparison} ($\mathtt{X\_SafeComp}$). We could formalise this perspective in our language as:
\begin{equation}
\spbox{\mathsf{IT}} (\mathtt{X\_Safe} \rightarrow \mathtt{X\_SafeComp} \lor \mathtt{X\_TestSafe})
\end{equation}
 The notion of `safe by comparison' relies on $\mathtt{X}$'s relation to other devices within its domain of application. If a \emph{comparable} device $\mathtt{Y}$ exists in terms of architecture, materials used, applicability, etc. ($\mathtt{X\_Comp\_Y}$), and this device has been deemed safe by testing ($\mathtt{Y\_TestSafe}$), then it can be argued that $\mathtt{X}$ is safe by comparison. We can express this as the formula shown below, where $\mathrm{Devices}$ consists of the set of devices within $\mathtt{X}$'s domain of application (of which there are only finitely many).
\begin{equation}
\spbox{\mathsf{IT}}(\mathtt{X\_SafeComp} \rightarrow \!\!\!\!\!\!\!\! \bigvee_{\mathtt{Y} \in \mathrm{Devices}} \!\!\!\!\!\!\!\! \mathtt{X\_Comp\_Y} \land \mathtt{Y\_TestSafe})
\end{equation}
 In contrast to Italy's perspective, Germany may qualify $\mathtt{X}$ as safe, so long as it has been tested safe, that is:
\begin{equation}
\spbox{\mathsf{DE}} (\mathtt{X\_Safe} \rightarrow \mathtt{X\_TestSafe})
\end{equation}
 We can see that Germany's standpoint on what counts as safe is more stringent than, or subsumed by, Italy's standpoint, a fact which may be encoded as $\mathsf{DE} \sharper \mathsf{IT}$. As each nation takes a different stance on what it deems safe, it is clear that propositions may be inconsistent with one perspective as opposed to the other. For example, the proposition `it is conceivable that $\mathtt{X}$ is safe overall though not safe according to testing', i.e. the formula
$\spdia{\mathsf{\dst}} (\mathtt{X\_Safe} \land \neg \mathtt{X\_TestSafe})$ is consistent with formulae (1)-(4) above. This is due to Italy's standpoint, which allows for $\mathtt{X}$ to be safe by comparison. Still, the formula $\spbox{\mathsf{\dst}} (\mathtt{X\_Safe} \land \neg \mathtt{X\_TestSafe})$ is inconsistent with formulae (1)-(4) since Germany regards $\mathtt{X}$ as not safe overall if it has not been determined safe by testing.

As demonstrated above, the use of standpoint modalities permits the modelling of distinct, conflicting perspectives. Without the use of modalities to `wrap' reasoning, we would achieve undesirable inconsistencies that are unrepresentative of the actual (consistent) scenario we aim to model. Thus, the use of standpoint modalities with LTL improves our capacity to represent knowledge more faithfully, and as shown in \sect~\ref{sec:complexity}, this does not come at a cost in terms of complexity. Indeed, an attractive feature of standpoint logic is its low complexity in relation to other knowledge integration approaches, being $\np$-complete~\cite{gomez2021standpoint,LyoGom22}.

 \subsection{Semantics}
  

 The models used for $\sltl$ are variants of relational models, employing \emph{state sequences} (as in $\ltl$) as opposed to worlds and \emph{standpoint interpretations} rather than accessibility relations; cf.~\cite{gomez2021standpoint}.

\begin{definition}[Temporal Standpoint Structure] Let $\vocab = \langle \pset, \spset \rangle$ be a vocabulary. We define a \emph{state sequence} (relative to $\vocab$) to be an infinite sequence of states $\bar\sigma=\tuple{\sigma_0,\sigma_1,\ldots}$, where each state $\sigma_i \subseteq \pset $. A \emph{temporal standpoint structure} is an ordered-pair $\model = \tuple{\Pi,\lambda}$ such that $\Pi$ is a non-empty set of state sequences $\tuple{\sigma_0,\sigma_1,\ldots}$ and $\lambda:\spset \to 2^\Pi$ is a \emph{standpoint interpretation} mapping each standpoint symbol to a non-empty subset of $\Pi$ with $\lambda(\ast) = \Pi$. 



\end{definition}

\begin{definition}[Satisfaction]\label{def:semantic-clauses} Let $\phi \in \langsltl$ and $\model = \tuple{\Pi,\lambda}$ be a temporal standpoint structure such that $\bar\sigma \in \Pi$. We recursively define the \emph{satisfaction} of $\phi$ on $\bar\sigma$ at the time point $n\ge0$, written $\bar\sigma,n\models\phi$, as follows:
\begin{itemize}
\item $\model, \bar\sigma, n \models \sp\preceq\st$ \iffi  $\lambda(\sp)\subseteq\lambda(\st)$;

\item $\model, \bar\sigma, n \models p$ \iffi  $p\in\sigma_n$;

\item $\model, \bar\sigma, n \models \neg p$ \iffi  $p \not\in\sigma_n$;

\item $\model, \bar\sigma, n \models \phi \lor \psi$ \iffi $\model, \bar\sigma, n \models \phi$ or $\model, \bar\sigma, n \models \psi$;

\item $\model, \bar\sigma, n \models \phi \land \psi$ \iffi $\model, \bar\sigma, n \models \phi$ and $\model, \bar\sigma, n \models \psi$;

\item $\model, \bar\sigma, n \models \next \phi$ \iffi $\model, \bar\sigma, n+1 \models \phi$;

\item $\model, \bar\sigma, n \models \even \phi$ \iffi there is a $j > n$ such that $\model, \bar\sigma, j \models \psi$;

\item $\model, \bar\sigma, n \models \phi \rel \psi$ \iffi either $\model, \bar\sigma, i \models \psi$ for all $i \geq n$, or there exists a $k \geq n$ such that $\model, \bar\sigma, k \models \phi$ and $\model, \bar\sigma, j \models \psi$ for all $n \leq j \leq k$;

\item $\model, \bar\sigma, n \models \phi \until \psi$ \iffi there exists a $j \geq n$ such that $\model, \bar\sigma, j \models \psi$ and for every $n \leq i < j$, $\model, \bar\sigma, i \models \phi$;
    
\item $\model, \bar\sigma, n\models\spdia{\st} \psi$ \iffi for some $\bar\sigma'\in\lambda(\st)$, $\model, \bar\sigma', n\models\psi$;

\item $\model, \bar\sigma, n\models\spbox{\st} \psi$ \iffi for each $\bar\sigma'\in\lambda(\st)$, $\model, \bar\sigma', n\models\psi$;

\item $\model,\bar\sigma\models\phi$ \iffi $\model,\bar\sigma,0\models\phi$;

\item $\model\models\phi$ \iffi for all $\bar\sigma\in \Pi$ then $\model,\bar\sigma\models\phi$ .
\end{itemize}

We say that that a set of formulas $\Phi$ is \emph{valid}, written $\models\Phi$, \iffi $\model\models\phi$ for each $\phi\in\Phi$ and each temporal standpoint structure $\model$. For a vocabulary $\vocab$, we define the \emph{standpoint logic} $\sltlv := \{\phi \in \langsltl \ | \ \models \phi\}$.
\end{definition}

\section{Automating Reasoning via Tableaux}\label{sec:tableaux}


\begin{table}[t]
\begin{center}
\bgroup
\def\arraystretch{1.5}
\begin{tabular}{r @{\hskip 1em} c @{\hskip 1em} l @{\hskip 1em} l}
\hline
$\ru$ &\multicolumn{1}{c}{$\phi\subseteq\Delta$ for} & \multicolumn{1}{c}{$\Gamma_{1}(\phi)$} & \multicolumn{1}{c}{$\Gamma_{2}(\phi)$}\\
  &  $\setindex \Delta \in \Gamma$ &  & \\

\hline
$\disr$ & ${\alpha \lor \beta}$ & $\nest{\alpha}$ & $ \nest{\beta}$\\
$\untr$ & $ {\alpha\,\until \beta}$ & $\nest{\beta}$ & $\nest{\alpha, \next(\alpha\, \until \beta)}$\\
$\relr$ & $ {\alpha \rel \beta}$ & $\nest{\alpha,\beta}$ & $\nest{\beta, \next(\alpha \rel \beta)}$\\
$\evenr$ & ${\even \alpha}$ & $\nest{\alpha}$ & $\nest{\next \even \alpha}$\\
$\conr$ & $ {\alpha \land \beta}$ & $\nest{\alpha,\beta}$ & \ \\
$\alwr$ & $ {\alw \alpha}$ & $\nest{\alpha,\next \alw \alpha}$ & \ \\
\hline\hline
\end{tabular}

\egroup
\end{center}
\caption{Expansion rules 1: When a singleton $\phi$ of the types shown in the table is found in one indexed set $\setindex \Delta$ of the set $\nodei{u}$ of a node $u$, one or two children nodes $u'$ and $u''$ are created, each with a copy $\nodei{u'}$ and $\nodei{u''}$ of  $\nodei{u}$ in which $\phi$ has been replaced by $\Gamma_1(\phi)$ and $\Gamma_2(\phi)$ (respectively) in the set indexed $\setindex\Delta$.}
\label{fig:expansion-rules}
\end{table}

This section presents the tableau calculus for
$\sltl$, which we will later show to be terminating, sound, and complete. 
The tableau introduces support for standpoint reasoning into the structure of a tableau calculus for $\ltl$. This is done by encapsulating a full set of standpoint interpretations (at a given time) within each node and exploiting \emph{quasi-model}  based techniques \cite{Wolter98SatisfiabilityDescriptionLogicsModalOperators}. In particular, we
extend the tree-shaped tableau by Reynolds~\cite{Reynolds16a}, which proved to be quite amenable for extensions of LTL to
different logics~\cite{GeattiGMR21} and for efficient
implementations~\cite{GeattiGM19}. Classic tableaux for
$\ltl$~\cite{LichtensteinP00} usually build a graph structure and then, in a
second pass, look for models inside it. In contrast, Reynolds' tableau builds a
tree structure where each branch can be  explored independently from the others
in a single pass (which also aids parallelization~\cite{McCabeDanstedR17}).
Furthermore, its modular rule-based structure and it's model-theoretic
completeness proofs~\cite{GeattiGMR21} are easy to extend to different logics.

Here, we exploit these features to obtain a tree-shaped tableau calculus for
$\sltl$. We remark that the overall structure of Reynolds' tableau remains unchanged, but the
nodes' labels are enriched, and additional rules are devised, to deal with
standpoint modalities. Let us begin by presenting a sequence of useful definitions.

\begin{definition}[Closure] Let $\Phi \subseteq \langsltl$. The closure $\clo{\Phi}$ is the set defined as follows
\begin{enumerate}

\item For every $\phi\in\Phi$ and $\psi\in\sufo{\phi}$, $\psi\in\clo{\Phi}$;

\item For every $\neg p\in\pset$, $p \in \clo{\Phi}$ if and only if $\neg p \in \clo{\Phi}$;

\item If $\phi_1\!\circ\!\phi_2\! \in\! \clo{\Phi}$, for $\circ\!\in\!\nest{\mathcal{U,R}}$, then $\next(\phi_1\! \circ\! \phi_2)\!\in\! \clo{\Phi}$;

\item If $\triangledown\phi\in \clo{\Phi}$, for $\triangledown\in\nest{\alw,\even}$, then $\next\phi\in \clo{\Phi}$.

\end{enumerate}
\end{definition}

Intuitiely, the closure of a formula is the set of all the formulas that is
sufficient to take care of when reasoning about the formula. A tableau for
$\Phi$ is a tree made of sets of constraint sets, defined below. Each constraint set is associated with the set of standpoints to which it belongs and with the set of formulas (from the closure) that hold for that point.

\begin{definition}[Standpoint Encoding]
 The encoding of a standpoint $\st$ in $\Phi$ is defined as
$\stenc{\st}=\{\st\} \cup \{\sp \mid \st\preceq\sp \in \Phi\}$. We use $\stence$, $\stence'$, \dots to refer to standpoint encodings. We let $E=\{\stenc{\st}\mid \st \in \Phi\}$ and also $\stence\preceq\stence'$ iff $\stence'\subseteq\stence$.
\end{definition}

\begin{definition}[Constraint Set]
 Let $\cs, \cs', \cs'', \dots$ be \emph{constraint sets} of the form $\constraintset$ where:
 \begin{itemize}
     \item $\stence$ is a standpoint encoding;
     \item $\ell$ is a label (that may be empty, in which case $\ell=\false$);
     \item $\Delta = \{\phi_{1}, \ldots, \phi_{n}\}$ is a subset of $\clo{\Phi}$.
 \end{itemize}
 We use the shortcut functions $\csenc{\cs_i}=\stence_i$, $\cslab{\cs_i}=\ell_i$ and $\csfor{\cs_i}=\Delta_i$, for $\cs_i=\langle \stence_i, \ell_i, \Delta_i\rangle$. Finally, we use $\setid{\cs}$ to denote to a constraint set such that $\cslab{\cs}\neq\false$ (referred to as a \emph{labelled or diamond set}) and $\setib{\cs}$ to denote a constraint set such that $\cslab{\cs}=\false$ (referred to as an \emph{unlabelled or box set}).
\end{definition}

\begin{definition}[Node Constraint Set for $\Phi$]
Let $u, u', u'', \dots$ be the nodes of a tableau checking the satisfiability of a set of formulas $\Phi$. The constraint set of a node $u$, denoted $\nodei{u}$, is a set $\nodei{u} := \{ \cs_1,\ \mydots,\ \cs_n \}$ of constraint sets such that
\begin{description}
    \item[(a)] For each $\st\in\spset$, there is exactly one constraint set $\setib{\cs}\in\nodei{u}$ such that $\csenc{\cs} = \stencF{\st}$ (i.e., there is one unlabelled set per standpoint);
    \item[(b)] For each $\setid{\cs}\in\nodei{u}$, there is no $\cs'\in\nodei{u}$ with $\cslab{\cs}=\cslab{\cs'}$ (i.e., labels are unique).
\end{description}
 We say that $\nodei{u}=\nodei{u'}$ if for each $\cs\in\nodei{u}$ there is some $\cs'\in\nodei{u'}$ such that $\csenc{\cs}=\csenc{\cs'}$ and $\csfor{\cs}=\csfor{\cs'}$ and \viceversa.
\end{definition}

Intuitively, $\nodei{u}$ represents the structure of standpoints at a certain
time point. Later, we will show how to construct a model $\model$ from a tableau
branch. In doing that, each labelled set $\cs$ will correspond to a point in a sequence belonging to the standpoint determined by
$\csenc{\cs}$. Unlabelled sets encode `standpoint
types' and are necessary to ensure that full precisification sequences can be reconstructed during model construction.

Let us define some useful functions. We will use $\iota(\Gamma,\stence,\ell) = \csfor{\cs}$, if there is some $\cs \in \Gamma$ such that $\csenc{\cs}=\stence$ and $\cslab{\cs}=\ell$; otherwise, $\iota(\Gamma,\stence,\ell) = \emptyset$. 
The \emph{union} and \emph{set difference} operations are defined as follows (respectively) for sets of constraint sets: 
\begin{itemize}
    \item $\Gamma \cup \Gamma' = \{\ \cslong{\stence,\ell,\Delta \cup \Delta'} \ | \ \constraintset\in \Gamma,  \iota(\Gamma',\stence,\ell)=\Delta'\}$; 
        \item $\Gamma \setminus \Gamma' = \{\ \cslong{\stence,\ell,\Delta \setminus \Delta'} \ | \ \constraintset\in \Gamma,  \iota(\Gamma',\stence,\ell)=\Delta'\}$.
\end{itemize}
 
 Branches within a tableau will either become \emph{ticked} ($\ticked$) or \emph{crossed} ($\crossed$) throughout the processing and expansion of the tableau. Once all branches are ticked or crossed, the tableau is deemed \emph{complete}, and a model can be extracted from each ticked branch. A counter-model can be extracted if all branches are crossed.
 
 An \emph{$\next$-eventuality} is defined to be a formula of the form $\next(\phi\, \until\, \psi)$, and if such a formula occurs in the constraint set of a node, this implies that a pending request still needs to be fulfilled at a future moment. A \emph{poised branch} is defined to be a branch $\branchdef$ such that $u_{n}$ contains only constraint sets with formulae of the form $p$ and $\neg p$ with $p \in \pset$ and $\next \alpha$ with $\alpha \in \langsltl$, and such that no expansion rules are applicable. Let us now describe how our tableaux are initialised and how rules are applied to them.
 
\textbf{Initialisation:} Tableaux are generated by taking $r$ as the root node with $\nodei{r}= \nest{\cslong{\nest{*},\ell_0,\Phi}}\cup\{\cslong{\stenc{\st},\false,\emptyset} \mid \st\in\spset\}$
as input.

\textbf{Expansion:} We expand a tableau by repeatedly applying the rules in \fig~\ref{fig:expansion-rules} as well as the rules defined below. 

The \emph{expansion rules}, shown in \fig~\ref{fig:expansion-rules}, work as follows: each rule looks for constraint set $\constraintset\in\nodei{u}$ and a formula $\phi \in\Delta$, where $u$ is the current node. By applying the rules, one or two children $u'$ and $u''$ are created with $\nodei{u'} = \nodei{u} \setminus \nest{\cslong{\stence,\ell,\nest{\phi}}} \cup \Gamma_1(\phi)$ and $\nodei{u''} = \nodei{u} \setminus \nest{\cslong{\stence,\ell,\nest{\phi}}} \cup \Gamma_2(\phi)$, respectively. In addition, we have the following rules:

\begin{description}

\item[$\sboxri$] 
Given a node constraint set $\nodei{u}$ and a constraint set $\cs\in\nodei{u}$ with $\standb{\st}\alpha\in \csfor{\cs}$ and $\bot\notin \csfor{\cs}$, a child node $u'$ is created with
$\nodei{u'}=\nodei{u'}\setminus\nest{\cslong{\stence,\ell,\nest{\standb{\st}\alpha}}}\cup\nest{\cslong{\stencF{\st},\false,\nest{\alpha}}}$, and if $\cslab{\cs}=\false$, then a second child node is created with $\nodei{u'}=\nodei{u'}\setminus\nest{\cslong{\stence,\ell,\nest{\standb{\st}\alpha}}}\cup\nest{\cslong{\stence,\false,\nest{\bot}}}$.

\item[$\sboxrii$] 
Given a node constraint set $\nodei{u}$ and two constraint sets $\setib{\cs},\cs'\in\nodei{u}$ with $\csenc{\cs'}\supseteq\csenc{\cs}$, $\alpha\in \csfor{\cs}$, and $\alpha\notin\csfor{\cs'}$, then a child node $u'$ is created with
$\nodei{u'}=\nodei{u'}\cup\nest{\cslong{\csenc{\cs'},\cslab{\cs'},\nest{\alpha}}}$.

\item[$\sdiar^1$] 
Given a node constraint set $\nodei{u}$ and an indexed set $\constraintset\in\nodei{u}$ with $\standd{s}\alpha\in\Delta$, then a child node $u'$ is created such that $\nodei{u'}=\nodei{u}\setminus\nest{\cslong{\stence,\ell,\nest{\standd{s}\alpha}}}\cup\nest{\cslong{\stenc{\st},\ell_i,\nest{\alpha}}}$, with $\ell_i=\ell'$ if there is  $\cslong{\stenc{\st},\ell',\Delta'}\in\nodei{u}$ with $\alpha\in\Delta'$ and otherwise $\ell_i$ is a fresh label.

\item[$\sdiar^2$] 
Given a node constraint set $\nodei{u}$ and two indexed sets $\constraintset,\cslong{\stence,\ell',\Delta'}\in\nodei{u}$ with $\Delta'\subseteq\Delta$ and $\ell'\neq\false$, then a child node $u'$ is created with $\nodei{u'}=\nodei{u}\setminus\nest{\cslong{\stence,\ell',\Delta'}}$.

\item[$\stepr$] Given a poised branch $\branch{u} = \langle u_{0}, \ldots, u_{n} \rangle$, a child $u'$ is added to $\branch{u}$ with $\nodei{u'}= \nest{  \cslong{\stence,\false,\nest{\alpha \mid \next\alpha\in\Delta}} \mid\cslong{\stence,\false,\Delta}\in\nodei{u}} \cup \nest{  \cslong{\stence,\ell,\nest{\alpha \mid \next\alpha\in\Delta}} \mid\cslong{\stence,\ell,\Delta}\in\nodei{u}, \Delta\neq\emptyset}$
\end{description}

\begin{description}
\item[$\contrar$] If $\cs\in\nodei{u_{n}}$ and $\bot \in \csfor{\cs}$, then $u_{n}$ is \emph{crossed}.\footnote{$\bot$ is a shortcut for $\{p,\neg p\}$ for some $p \in \pset$.}
\end{description}

\begin{description}
\item[$\emptyr$] Given a branch $\branch{u} = \langle u_{0}, \ldots, u_{n} \rangle$, if for all $\cs\in\nodei{u_{n}}$ we have $\csfor{\cs} = \emptyset$, then $u_{n}$ is \emph{ticked}.
\end{description}

\begin{description}
\item[$\loopr$] If there exists a poised node $u_{i}$ such that $u_{i} < u_{n}$, $\nodei{u_{i}} = \nodei{u_{n}}$, and all $\next$-eventualities requested in $u_{i}$ are fulfilled in $\branch{u}_{[i+1 \ldots n]}$, then $u_{n}$ is \emph{ticked}.
\end{description}

\begin{description}
\item[$\pruner$] If two positions $i$ and $j$ exist such that $i < j \leq n$, $\nodei{u_{i}} = \nodei{u_{j}} = \nodei{u_{n}}$, and among the $\next$-eventualities requested in these nodes, all those fulfilled in $\branch{u}_{[j+1 \ldots n]}$ are fulfilled in $\branch{u}_{[i+1 \ldots j]}$ as well, then $u_{n}$ is \emph{crossed}.
\end{description}

\begin{figure}
\centering
    \begin{tikzpicture}[scale=1, every node/.style={font=\footnotesize}]
    \node[tableau node] (root) at (0,0) {\cslong{\{\dst\}, \ell_{0}, \nest{\spdia{\dst} (p & X!p), \spbox{\dst}{X p}}}}
      child[tableau edge] {
        node[tableau node]{
          \cslong{\{\dst\}, \ell_{0}, \nest{\spdia{\dst} (p & X!p)}},
          \cslong{\{\dst\}, \false, \nest{X p}}
        }
        child[tableau edge, sibling distance=2cm] {
          node[tableau node] { 
            \cslong{\{\dst\}, \ell_{0}, \emptyset},
            \cslong{\{\dst\}, \false, \nest{X p}},
            \cslong{\{\dst\}, \ell_{1}, \nest{p & X!p}}
          }
          child[tableau edge] {
            node[tableau node] {%
              \cslong{\{\dst\}, \ell_{0}, \emptyset},
              \cslong{\{\dst\}, \false, \nest{X p}},
              \cslong{\{\dst\}, \ell_{1}, \nest{p, X!p}}
            }
            child[tableau edge] {
              node[tableau node, xshift=-0mm] {%
                \cslong{\{\dst\}, \ell_{0}, \emptyset},
                \cslong{\{\dst\}, \false, \nest{X p}},
                \cslong{\{\dst\}, \ell_{1}, \nest{!p, X p, X!p}}
              }
                child[tableau edge,step rule] {
                node[tableau node, xshift=-0mm, label=below:\crossed] {%
                   \cslong{\{\dst\}, \ell_{0}, \emptyset},
                   \cslong{\{\dst\}, \false, \nest{p}},
                   \cslong{\{\dst\}, \ell_{1}, \nest{p, !p}}
                }
              }
            }
          }
        }
      };


      \def\underline(#1){%
        \draw ($(#1.south west)+(0.15,0.1)$) --
              ($(#1.south east)+(-0.15,0.1)$)
}

\end{tikzpicture}
  %
\caption{We provide an example demonstrating a run of our tableau algorithm with the input $\{\spdia{\dst} (p \land \next \neg p), \spbox{\dst} \next p\}$. The rules applied (in order) are $\sboxri$, $\sdiar^1$, $\conr$, $\sboxrii$, $\stepr$, and $\contrar$, thus showing that the tableau creates a crossed branch, i.e. the input set is unsatisfiable.\label{fig:tableaux:examples}}
\end{figure}

\section{Soundness and Completeness}\label{sec:proofs}


We now prove soundness, completeness, and termination of the tableau calculus described above. The proofs are inspired by \cite{GeattiGMR21} and adapted introducing the \emph{quasi-model} infrastructure \cite{Wolter98SatisfiabilityDescriptionLogicsModalOperators} when needed, seldom referring to existing proofs while being as self-contained as possible.

We start with the definition of \emph{pre-models}, which are structures that
summarize a set of models of a formula. We will see that each branch in a
tableau is associated with a pre-model, and \viceversa. We start with the notion
of \emph{atom}, which is a consistent set of formulas from the closure.

\begin{definition}[Atom] An \emph{atom} for a set of formulas $\Phi \subseteq \langsltl$ is a set of formulas $\Delta$ such that
  \begin{enumerate}
  \item $\Delta \subseteq \clo{\Phi}$;
  
  \item $\bot\notin \Delta$;
  
  \item If $\psi \in \Delta$, then either $\Gamma_{1}(\psi) \subseteq \Delta$, or $\Gamma_{2}(\psi)\neq\emptyset$ and $\Gamma_{2}(\psi) \subseteq \Delta$, where $\Gamma_{1}(\psi)$ and $\Gamma_{2}(\psi)$ are defined as in \fig~\ref{fig:expansion-rules};
  
  \item For each $\psi, \psi' \in \clo{\Phi}$, if $\psi \in \Delta$ and $\psi \Vdash \psi'$, then $\psi' \in \Delta$, i.e., $\Delta$ is
  closed by logical deduction (as far as the closure is concerned);
  
  \end{enumerate}
\end{definition}

Atoms are assigned to standpoint encodings and collected into
\emph{timestamps}.  
  
\begin{definition}[Timestamp]\label{def:timestamp}
  We define a \emph{timestamp} $\tstamp=\{ a_1,\  \dots, a_n \} $
  to be a set of \emph{indexed atoms}, that is, each $a_i$ is of the form $\langle \stence_i, \ell_i, \Delta_i\rangle$, with $\stence_i$ a standpoint encoding, $\ell_i$ a label, and ${\Delta_i}$ an atom such that
  
  \begin{description}
    \item[T1] For each $a \in \tstamp$, if $\standds \psi \in \csfor{a}$, then there exists a $a'\in \tstamp$ such that $\csenc{a'}\preceq\stencF{\st}$ and $\psi \in \csfor{a'}$;

    \item[T2] For each $a \in \tstamp$, if $\standbe \psi \in \csfor{a}$, then $\psi \in \csfor{a'}$ for each $a'\in \tstamp$ with $\csenc{a'}\preceq\stencF{\st}$;

    \item[T3] For each $\stence\in E$ with $E=\{\stenc{\st}\mid \st\in\spset\}$, there is $a_{\stence} \in \tstamp$ such that $\csenc{a_{\stence}}=\stence$ and $\csfor{a_{\stence}}=\bigcap\{\csfor{a'}\mid a'\in\tstamp, \csenc{a}\preceq\stence\}$

  \end{description}
  As with the constraint sets, we use the shortcut functions $\csenc{a_i}=\stence_i$, $\cslab{a_i}=\ell_i$ and $\csfor{a_i}=\Delta_i$. Moreover, we say that $a_i=a_j$ iff $\csenc{a_i}=\csenc{a_j}$ and  $\csfor{a_i}=\csfor{a_j}$.
\end{definition}

Timestamps are collected in sequences, and \emph{runs} map each point in time with an atom of the respective timestamp.

\begin{definition}[Run]\label{def:run}
  Let $\tsec=\langle\tstamp_0,\tstamp_1,\dots\rangle$ be an infinite sequence of timestamps. A run $\run$ on $\tsec$ is a map associating each $i\in\mathbb{N}$ with an indexed atom $a\in\tstamp_i$ in such a way that:
  \begin{description}
    \item[R1] For all $a_i=\run(i)$ and $a_j=\run(j)$, $\csenc{a_i}=\csenc{a_j}$.

    \item[R2] If $\next \psi \in \csfor{a_i}$ for $a_i=\run(i)$, then $\psi \in \csfor{a_{i+1}}$ for $a_{i+1}=\run(i+1)$;

    \item[R3] If $\psi_{1} \until \psi_{2} \in \run(i)$, then there is a $j \geq i$ with $\psi_{2} \in \csfor{a_j}$ for $a_j= \run(j)$ and $\psi_{1} \in  \csfor{a_k}$ for $a_k=\run(k)$ for all $i \leq k < j$;
        
    \end{description}
  \end{definition}

  Now, a pre-model is made of a sequence of timestamps and a set of runs that
  encode the sequences of a usual model.

  \begin{definition}[Pre-Model] \label{def:pre-model}
  Let $\premd=\langle\tsec,\runs\rangle$ be a tuple consisting of an infinite sequence of timestamps $\tsec$ and a set of runs $\runs$ such that for every timestamp $\tstamp_i$ and every atom $a\in\tstamp_i$ there is a run such that $\run(i)=a$. $\premd$ is a \emph{pre-model} of $\Phi \subseteq\langsltl$ if the following constraints are satisfied:
  \begin{description}

  \item[P1] There exists an $a\in \tsec(0)$ such that $\csenc{a}=\stenc{*}$ and $\Phi \subseteq \csfor{a}$;

  \item[P2] For each $i \in \mathbb{N}$ and $a \in \tstamp_i$, $\csfor{a}$ is \emph{minimal}, i.e. for any $\psi \in \csfor{a}$, the set $\csfor{a} \setminus\nest{\psi}$ is (a) not an atom or (b) $\tstamp$ ceases to be a timestamp or (c) there ceases to be a run.

  \end{description}
    
\end{definition}

The core feature of pre-models is their correspondence with actual models of formulae, stated in the following.

\begin{restatable}{lemma}{pre-modeliffmodel}\label{lem:pre-model-implies-sat}
  Let $\phi \in \langsltl$, then $\phi$ has a pre-model iff $\phi$ is satisfiable.
\end{restatable}
\begin{proof}[Proof sketch]
  The proof is articulated but straightforward, following the definitions of
  pre-models and the semantics of $\sltl$. See \cref{app:proofs} for the full
  proof.
\end{proof} 

With the concept of pre-models we can now prove soundness and completeness.

\subsection{Soundness}

Here we prove that the tableau system is sound, that is, if a complete
tableau for a set of formulae $\Phi$ has a successful branch, then $\Phi$ is satisfiable. 
The proof shows how a pre-model for $\Phi$ can be extracted
from a successful branch of a complete tableau. Intuitively, the expansion of non-poised
nodes builds the set of constraint sets that will form the current timestamp,
and the application of the $\stepr$ rule to a poised node marks the advancement
to the next one.


\begin{definition}[Step node]
Consider a branch $\branchdef$ of a complete tableau $T$. A poised node $u_i$ is said to be a step node if either $u_i=u_n$ or $u_{i+1}$ is the child of $u_i$ added by applying the $\stepr$ rule.
\end{definition}

Only step nodes will be considered when looking at a given tableau branch to build the corresponding pre-model of $\Phi$. Now we can state how exactly to perform this extraction. Let us first define how single timestamps and their atoms are built from each step node.



\begin{definition}[Atom of a tableau node]
Let $\Phi \subseteq \langsltl$, $u_i$ be a step node of a tableau for $\Phi$ and $\cs\in\nodei{u_i}$. 
The atom extracted from $\cs$, written $\csatomf{\cs}$, is the indexed atom $\csatom$ such that $\csenc{\csatom}=\csenc{\cs}$ and $\csfor{\csatom}=\Delta$ where $\Delta$ is the closure by logical entailment (within $\clo{\Phi}$) of $\csfor{\cs}$, $\Delta= \cl{\csfor{\cs}}$.
\end{definition}

\begin{definition}[Timestamp of a tableau node]
Let $u_i$ be a step node of a tableau for $\Phi$. 
The timestamp extracted from $\nodei{u_i}$ is defined as $\csatomf{\nodei{u_i}}=\{\csatomf{\cs}\mid \cs\in\nodei{u_i}\}$
\end{definition}



Let us now define how to extract a pre-model from a successful branch of a complete tableau.

\begin{lemma}\label{lem:success-implies-pre-model}
 Let $\Phi \subseteq \langsltl$ and let $T$ be a complete tableau for $\Phi$. If $T$ has a successful branch, then there exists a pre-model for $\Phi$.
\end{lemma}

\begin{proof} 

Let $\branchdef$ be a successful branch of $T$ and let $\branchvdef$ be the subsequence of step nodes of $\branch{u}$. Intuitively, a pre-model for $\Phi$ can be obtained from $\branch{v}$ by building the atoms from the labels of the step nodes, and extending them to infinite sequences. If the branch has been accepted by the $\loopr$ rule, we can identify a position $0 \leq k \leq m$ in $\branch{v}$ such that $\nodei{v^{k}}=\nodei{v^m}$ 
and all the X-eventualities requested in $v^{k}$ are fulfilled in $v_{[k\dots m]}$. If instead $\branch{v}$ has been accepted by the $\emptyr$ rule and in particular there are no X-eventualities requested, hence setting $k = m$ we obtain the same effect. 

Therefore, we can extract from $v$ the periodic sequence of timestamps $\tsec= \tsec_0\tsec^\omega_t $, where $\tsec_0 = \langle \csatomf{\nodei{v_0}}, \dots, \csatomf{\nodei{v_k}} \rangle$, and either $\tsec^\omega_t  = \langle \csatomf{\nodei{v_k+1}}, \dots, \csatomf{\nodei{v_m}} \rangle$ or $\tsec^\omega_t  = \langle \csatomf{\nodei{v_m}} \rangle$ depending on which rule accepted the branch, respectively the $\loopr$ or the $\emptyr$ rule. In other words, we build a periodic pre-model that infinitely repeats the fulfilling loop identified by the $\loopr$ rule, or the last empty node otherwise. Then let $K : \mathbb{N} \rightarrow \mathbb{N}$ be the map from positions in the pre-model to their original positions in the branch, which is defined as $K(i) = i $ for $0 \leq i < k$, and for $i \geq k$ is defined either as $K(i)=k+((i-k)\, \mathrm{mod}\, T)$, with $T=m-k$ ($\loopr$ rule), or as $K(i) = k$ ($\emptyr$ rule).

Then, we set $\runs$ to be the set of all possible runs on $\tsec$ and we show that $\premd=\langle\tsec,\runs\rangle$ is a pre-model of $\Phi$. Notice that by the initialisation of the tableau we have $\Phi\in \csfor{a_0}$ for some $a_0\in\tsec(0)$ and atoms are minimal by definition.

First, we show that the three conditions of timestamp in \Cref{def:timestamp} are satisfied for $\tsec(i)$ with $i\geq 1$. It is easily checked that T1 is satisfied by the non-applicability of $\sdiar^1$, T2 by the non-applicability of both $\sboxri$ and $\sboxrii$ and T3 by the non-applicability of $\sdiar^2$. 

Let us now show that each $\run\in\runs$ fulfils the conditions in \Cref{def:run} of a run for $i\geq 1$. For condition R1, we notice that for every $\stence\in E$ there is some $a\in\tsec(i)$ by the construction of the tableau, since during initialisation one box-indexed atom is created per standpoint expression and such atoms are always propagated by the $\stepr$ rule. Hence for each $i\geq 1$ there is some $a_j\in\tsec(j)$ such that $\run(i)=a$ and $\run(j)=a_j$ and $\csenc{a}=\csenc{a_j}$.

Assume now that $\next\psi\in\csfor{\run(i)}$. Being an elementary formula, we can observe that we must have $\next\psi\in\csfor{\cs}$ for some $\cs\in\nodei{u_i}$. Two cases have to be considered. If $v_{K(i+1)} = v_{K(i)+1}$, i.e., the next atom comes from the actual successor of the current one in the tableau branch, then, by the $\stepr$ rule, there is $a_{i+1}\in\tsec(i+1)$ such that $\csenc{a_{i+1}}=\csenc{a}$ and $\psi\in\csfor{a_{i+1}}$. Otherwise, $\tsec(i)=\tsec(m)=\csatomf{\nodei{v_m}}$, and $v_m$ was ticked by the $\loopr$ rule (since at least $\csfor{a}$ is not empty), and thus $\tsec(i+1)=\csatomf{\nodei{v_k+1}}$ for some $k<m$ such that $\nodei{v_k}=\nodei{v_m}$. Hence $\next\psi\in\csfor{\cs'}$ with $\cs'\in\nodei{v_k}$, and by the step rule applied to $v_k$ we obtain that $\psi\in\csfor{\cs''}$ with $\cs''\in\nodei{v_k+1}$, hence  there is $a''_{i+1}\in\tsec(i+1)$ such that $\csenc{a''_{i+1}}=\csenc{a}$ and $\psi\in\csfor{a''_{i+1}}$ as well in this case, hence condition R2 is also satisfied. Finally, the case of $\psi_1\until\psi_2\in\csfor{\run(i)}$ (condition R3) is then straightforward in view of the expansion rules definition.
With this, we have shown that a run can be created for each $a\in\tstamp_i$ and $i\geq 0$ satisfying all the conditions, and hence $\premd$ is a pre-model.
\end{proof}

The above results let us conclude the soundness of the tableau system.

\begin{theorem}[Soundness]
 Let $\Phi \subseteq \langsltl$ and let $T$ be a complete tableau for $\Phi$. If $T$ has a successful branch, then $\Phi$ is satisfiable.
\end{theorem}

\begin{proof}
Extract a pre-model for $\Phi$ from the successful branch of T as shown in
\Cref{lem:success-implies-pre-model}, and then obtain from it an actual model for the formula as shown
by \Cref{lem:pre-model-implies-sat}.
\end{proof}

\subsection{Completeness}

We now prove the completeness of the tableau system, i.e., if a set of formulae $\Phi$
is satisfiable, then any complete tableau T for it has an accepting branch. The
proof uses a pre-model for $\Phi$, which we know exists if the formula is
satisfiable, as a guide to suitably descend through the tableau to look for an
accepted branch. We first describe how to perform such a descent. Then, we will
show how to make sure that this descent must obtain an accepted branch. The
descent is performed as follows.


\begin{lemma}[Extraction of the branch]\label{lemma:branch-extraction}
 Let $\premd=\langle\tsec,\runs\rangle$ be a pre-model for a set of formulae $\Phi$. Then, any complete tableau $T$ for $\Phi$ has a label mapping $\mlabelf$ and a branch $u$, with a sequence of step nodes $\branchvdef$, such that for $0 \leq i \leq m$ and $a \in \tsec(i)$, we have one $\cs\in\nodei{v_i}$ such that $a=\csatomf{\cs}$.
\end{lemma}

\begin{proof}
Let $\premd=\langle\tsec,\runs\rangle$ be a pre-model. 
To find $\branch{u}$, we traverse the tree using $\tsec$ as a guide, starting from the root $u_0$, building a sequence of branch prefixes $\branch{u_i} = \langle u_0,\dots, u_i\rangle$, suitably choosing $\branch{u_{i+1}}$ at each step among the children of $\branch{u_i}$. We maintain a non-decreasing function $J : \mathbb{N} \rightarrow \mathbb{N}$ that maps positions in $\branch{u_i}$ to positions in $\tsec$ such that for each $a \in \tsec(J(k))$, there is some $\cs\in\nodei{u_k}$ such that $\csenc{\csatom}=\csenc{a}$ and   $\csfor{\csatom}\subseteq\csfor{a}$ with $\csatomf{\cs}=\csatom$, for each $0 \leq k \leq i$. Moreover, we maintain a labelling function $\mlabelf$ witnessing that relation,  $\mlabel{\cs}=a$
We start from $u_0 = \langle u_0\rangle$ and $J(0) = 0$. Let
$\cs_0=\cslong{\nest{*},\ell_0,\Phi}$ Notice that the tableau is
initialised with 
$$\nodei{r}= \nest{\cs_0}\cup\{\cs_\stence=\cslong{\stence,\false,\emptyset} \mid \stence\in E\}$$
Moreover, we have $\Phi \subseteq \csfor{a_0}$ for some $a_0 \in \tsec(0)$, and there is some $a_\stence \in \tsec(0)$ for each $\stence\in E$, by the definition of a pre-model. 
Thus, we set $\mlabel{\cs_{0}}=a_0$, and for each $\stence\in E$ we set $\mlabel{\cs_{\stence}}=a_{\stence}$. Thus, the base case clearly holds.

Then, at each step $i>0$, we choose $u_{i+1}$ among the children of $u_{i}$ as follows, depending on the expansion rule. We show that for each $a \in \tstamp(J(i))$, there is some $\cs\in\nodei{u_k}$ such that $\csenc{\csatom}=\csenc{a}$ and   $\csfor{\csatom}\subseteq\csfor{a}$ with $\csatomf{\cs}=\csatom$. We show the main cases and direct the reader to \cite{GeattiGMR21} for the rest.
\begin{description}
    \item[$\sboxri$] If $u_i$ is not a step node and was expanded by the $\sboxri$ rule then it has a single child which is chosen as $u_{i+1}$. We leave $\mlabelf$ unchanged and define $J(i + 1) = J(i)$, since we do not advance to the next position in the pre-model; We show that for the node $u_{i+1}$ the condition holds. By the definition of $\sboxri$ we have $\nodei{u_{i+1}}=\nodei{u_i}\setminus\nest{\cslong{\stence,\ell,\nest{\standb{e}\psi}}}\cup\nest{\cslong{\stencF{\st},\false,\nest{\psi}}}$. By the definition of a pre-model, if $\standb{e}\psi\in \csfor{a}$ for some $a\in\tsec(J(i))$, then $\psi \in \csfor{a'}$ for each $a'\in \tsec(J(i))$ with $\csenc{a'}\preceq\stencF{\st}$, and hence the condition will still be satisfied.

    \item[$\sboxrii$] If $u_i$ is not a step node and was expanded by the $\sboxrii$ rule then it has a single child which is chosen as $u_{i+1}$. We leave $\mlabelf$ unchanged and define $J(i + 1) = J(i)$, since we do not advance to the next position in the pre-model; We show that for the node $u_{i+1}$ the condition holds. By the definition of $\sboxrii$ we have $\nodei{u_{i+1}}=\nodei{u_i}\cup\nest{\cslong{\csenc{\cs'},\cslab{\cs'},\nest{\psi}}}$ for two constraint sets $\setib{\cs},\cs'\in\nodei{u_i}$ with $\csenc{\cs'}\preceq\csenc{\cs}$, $\psi\in \csfor{\cs}$, and $\psi\notin\csfor{\cs'}$. 
    And again by the definition of a pre-model, this means that $\psi\in \csfor{a_{\stence}}$ and thus for all  $a'\in \tsec(J(i))$ with $\csenc{a'}\preceq\stence$ then also $\psi \in \csfor{a'}$, and hence the condition will still be satisfied.

    \item[$\sdiar^1$] If $u_i$ is not a step node and was expanded by the $\sdiar^1$ rule then it has a single child which is chosen as $u_{i+1}$. We define $J(i + 1) = J(i)$ since we do not advance to the next position in the pre-model; 
    We show that for the node $u_{i+1}$ the condition holds. By the definition of $\sdiar^1$ we have $\nodei{u'}=\nodei{u}\setminus\nest{\cslong{\stence,\ell,\nest{\standd{s}\psi}}}\cup\nest{\cslong{\stenc{\st},\ell_i,\nest{\psi}}}$. 
    If $\ell_i$ is not a fresh label, then $\mlabelf$ is unchanged and $\nodei{u'}$ only removes $\standd{s}\psi$ from $\nodei{u}$, and hence the condition holds trivially. Assume $\ell_i$ is a fresh label. 
    By construction, we have $\cslong{\stence,\ell,\Delta}\in\nodei{u}$ with $\standd{s}\psi\in\Delta$, and a newly introduced $\cs'\in\nodei{u'}$ with $\cs'=\cslong{\stenc{\st},\ell_i,\nest{\psi}}$. 
    By induction, there is some $a\in\stseq(J(i))$ such that $\standd{s}\psi \in \csfor{a}$ and by the definition of a pre-model, there is some $a'\in\stseq(J(i))$ such that $\csenc{s}\preceq\stencF{\st}$ and $\psi \in \csfor{a'}$. Thus, we update $\mlabelf$ such that $\mlabel{\cs'}=a'$ and with this the condition is satisfied.

    \item[$\sdiar^2$]  If $u_i$ is not a step node and was expanded by the $\sdiar^2$ rule then it has a single child which is chosen as $u_{i+1}$. We define $J(i + 1) = J(i)$ since we do not advance to the next position in the pre-model; 
    We show that for the node $u_{i+1}$ the condition holds. By the definition of $\sdiar^2$ we have $\nodei{u'}=\nodei{u}\setminus\nest{\cslong{\stence,\ell',\Delta'}}$ with $\standd{s}\psi\in\Delta$ for $\constraintset,\cslong{\stence,\ell',\Delta'}\in\nodei{u}$ with $\Delta'\subseteq\Delta$ and $\ell'\neq\false$. Thus, we update $\mlabelf$ to the shrunk set and the condition is still satisfied. 
    
    \item[$\disr$, $\untr$, $\relr$, $\conr$, $\evenr$ $\alwr$] These are straightforward. Refer to Lemma 3 in \cite{GeattiGMR21}.

    \item[$\stepr$] If $u_i$ is a step node but not a leaf, then it has a single child which is chosen as $u_{i+1}$, leaving $\mlabelf$ unchanged and defining $J(i + 1) = J(i) + 1$ since we need to advance to the next position in the pre-model as well; We show that for the node $u_{i+1}$ the condition holds: For each atom $\csatom$ extracted from the tableau node, there is some $a \in \tsec(J(i+1))$ such that $\csenc{\csatom}=\csenc{a}$ and $\csfor{\csatom}\subseteq\csfor{a}$. By the definition of $\stepr$ we have  $\nodei{u'}= \nest{  \cslong{\stence,\false,\nest{\psi \mid \next\psi\in\Delta}} \mid\cslong{\stence,\false,\Delta}\in\nodei{u}} \cup \nest{  \cslong{\stence,\ell,\nest{\psi \mid \next\psi\in\Delta}} \mid\cslong{\stence,\ell,\Delta}\in\nodei{u}, \Delta\neq\emptyset}$. By the definition of a pre-model, if $\next\psi\in \csfor{a}$ for some $a\in\tsec(J(i))$ then $\psi\in \csfor{a'}$ for some $a'\in\tsec(J(i))$. Thus the condition follows.
\end{description}

Now, let $\branchdef$ be the branch found as described above, and let $\branchvdef$ be the sequence of its step nodes. Since the value of $J(i)$ is incremented only when an application of the $\stepr$ rule is traversed, for each $a \in \tsec(i)$, there is some $\cs\in\nodei{v_i}$ such that $\csenc{\csatom}=\csenc{a}$ and $\csfor{\csatom}\subseteq\csfor{a}$ with $\csatomf{\cs}=\csatom$. 
Moreover, for the minimality of the atoms in the pre-model, required by \Cref{def:pre-model}, we know that for all $a\in\tstamp_i$ then any formula $\psi\in\csfor{a}$ is either $\psi\in\Phi$ and $a=a_0$ or it has been obtained by the expansion of $\next$, $\standbs$ or $\standds$, or by the closure by logical entailment. Similarly, by the construction of the tableau, any $\phi\in\csatom$ has been obtained either by an application of a $\stepr$ fulfilling the $\next$-eventualities or by an application of $\sdiar^1$, $\sboxri$ or $\sboxrii$, fulfilling the standpoint eventualities, or by an expansion rule fulfilling the closure by logical entailment. Finally, $\sdiar^2$ fulfils the minimality criteria by deleting subsumed diamond-atoms. Hence we can conclude that for each $a \in \tsec(i)$, there is exactly one $\cs\in\nodei{v_i}$ such that $\csenc{\csatom}=\csenc{a}$ and $\csfor{\csatom}\subseteq\csfor{a}$ with $\csatomf{\cs}=\csatom$.
\end{proof}

The particular branch found as described above might, in general, be crossed.
However, it is immediate to note that it cannot possibly have been crossed by an
application of the $\contrar$ rule, since this would imply that the pre-model
itself is contradictory (which cannot happen because of
\cref{lem:pre-model-implies-sat}). Hence, if a crossed leaf is found, it has
been crossed by the $\pruner$ rule. We can, however, defin a particular class of
models (and their pre-models) such that when we descend through the tableau
following any model of this class, we cannot possibly find a node crossed by the
$\pruner$ rule, neither. This class is called \emph{greedy pre-models}. 

To define the notion of greedy pre-models, consider a pre-model
$\premd=\langle\tsec,\runs\rangle$. Let $X\subset\clo{\phi}$ be the set of all
the $\next$-eventualities in $\clo{\phi}$. For each run $\run\in\runs$ and each
$i\ge0$, the \emph{distance vector} $d^\run_i\in\N^X$ is a function mapping each
$\next$-eventuality $\psi$ to the distance $d^\run_i(\psi)$ of the first
position where it is fulfilled, if $\psi\in\run(i)$, or zero otherwise. For
example, if $\psi_1 \until \psi_2\in\run(4)$, and $\psi_2\in\run(7)$, then
$d^\run_4(\psi_1\until\psi_2)=3$. 

Two distance vectors are compared component-wise, \ie we define a partial order
such that $v\prec v'$ iff $v(\psi)<v'(\psi)$ for all $\psi\in X$. Then, we
compare two runs $\run$ and $\run'$ lexicographically, that is, $\run\prec\run'$
iff there is an $i\ge0$ such that $d^\run_j = d^{\run'}_j$ for all $j < i$ and
$d^\run_i \prec d^{\run'}_i$.

Intuitively, if $\run\prec\run'$, there is a point $i$ where $\run'$ uselessly
delays the fulfillment of an $\next$-eventuality, while $\run$ fullfils it
before, all the rest being equal before $i$.

\begin{definition}[Greedy pre-models]
  A run $\run$ for a formula $\phi$ is \emph{greedy} if there is no other
  $\run'$ such that $\run'\prec\run$. A pre-model $\premd$ for $\phi$ is greedy
  if all its runs are greedy.
\end{definition}

One can show that greedy pre-models actually exist.

\begin{lemma}[Limit of a sequence of runs]\label{lemma:greedy-pre-model-exists}
  Let $\run_1\succ\run_2\succ\dots$ be an infinite descending sequence of runs.
  Then, there exists a run $\run^\omega$ such that $\run^\omega\preceq\run_i$
  for all $i \geq 0$.
\end{lemma}

\begin{proof}[Proof sketch]
The proof is combinatorical, exploiting the properties of the lexicographic
ordeding between distance vectors and the fact that the closure is finite. It is
identical to the proof of Lemma 4 in \cite{GeattiGMR21}.
\end{proof}

\begin{lemma}[Existence of greedy pre-models]
    Let $\premd$ be a pre-model for a
formula $\phi$. Then, there is a greedy pre-model $\premd'\preceq\premd$.
\end{lemma}

\begin{proof}
A direct consequence of \cref{lemma:greedy-pre-model-exists}. See also Lemma 5 in
\cite{GeattiGMR21}.
\end{proof}

Now, we can connect greedy pre-models to the tableau, the $\pruner$ rule, and the
descent through the tree done in \cref{lemma:branch-extraction}. Consider a
pre-model $\premd=\langle \tsec,\runs \rangle$. We define the \emph{segment}
$\premd_{[j,k]}=\langle \tsec_{[j,k]},\runs_{[j,k]} \rangle$ as the result of
isolating the time points between $j$ and $k$. That is,
$\tsec_{[j,k]}=\langle\tsec(j),\ldots,\tsec(k)\rangle$ and for every
$\run\in\runs$, we have $\run'\in\runs_{[j,k]}$ where $\run'(i)=\run(i-j)$.

We say that a $\next$-eventuality $\psi_1\until\psi_2$ is \emph{requested} in
$\premd_{[j,k]}$ if there is a $\run\in\runs_{[j,k]}$ such that
$\psi_1\until\psi_2\in\run_{[j,k]}(0)$, and that is is \emph{fulfilled} in
$\runs_{[j,k]}$ if it is requested and $\psi_2\in\run_{[k,j]}(w)$ for some $0\le
w < (k-j)$.

Similarly, we define $\premd_{]j,k[}=\langle \tsec_{]j,k[},\runs_{]j,k[}
\rangle$ as the pre-model where the segment $\premd_{[j,k]}$ has been cut away.
That is, $\tsec_{]j,k[}=\langle \tsec(0),\ldots,\tsec(j),\tsec(k)\ldots\rangle$
and for all and only $\run\in\runs$, we have $\run'\in\runs_{[j,k]}$ where
$\run'(i)=\run(i)$ for $0\le i \le j$ and $\run'(w)=\run(w-(k-j))$ for $w > j$.

\begin{definition}[Redundant segments]
  \label{def:redundant-segment}
  Let $\premd=\langle \tsec,\runs \rangle$ be a pre-model for $\phi$ and let $ i < j < k$ be three positions such that $\tsec({i})=\tsec({j})=\tsec({k})$. Then, the segment $\premd_{[j+1...k]}$ of $\premd$ is redundant if not all the X-eventualities requested in $\premd_{[j+1...k]}$ are fulfilled in $\premd_{[j+1...k]}$, and all those fulfilled in $\premd_{[j+1...k]}$ are fulfilled in $\premd_{[i+1...j]}$ as well.
\end{definition}

We notice how \cref{def:redundant-segment} is similar to the definition of the
$\pruner$ rule, but focuses on the pre-model instead of the branch of the
tableau. The core feature of redundant segments, as the name suggests, is that
they can be removed from any pre-model, obtaining again a pre-model.

\begin{lemma}[Removal of redundant segments: correctness]\label{lemma:removal-correctness}
    Let $\premd=\langle \tsec,\runs \rangle$ be a pre-model for a formula $\phi$
    with a redundant segment $\premd_{[j+1...k]}$. Then, $\premd_{]j,k+1[}$ is a
    pre-model.
\end{lemma}
\begin{proof}
    We start by noting that since $\tsec({i})=\tsec({j})=\tsec({k})$, then
    $\run(i)=\run(j)=\run(k)$ for all $\run\in\runs$. By \cref{def:run}, for
    every $\next \psi \in \csfor{\run(j)}$, we have $\psi\in\csfor{\run(j+1)}$,
    hence also $\psi\in\csfor{\run(k+1)}$. Moreover, for any $\psi_1\until
    \psi_2\in\run(k)$, there is a $w\le k$ such that $\psi_2\in\csfor{\run(w)}$
    and $\psi_2\in\csfor{\run(\ell)}$ for all $k\le \ell < w$. Hence, note that
    in $\premd_{]j,k+1[}$, the conditions of \cref{def:run} on $\run(j)$ are
    satisfied. Hence $\premd_{]j,k+1[}$ is a pre-model.
\end{proof}

Now, we can observe that when a redundant segment is removed from a pre-model,
the runs in the resulting pre-models \emph{decrease} in the $\prec$ ordering.

\begin{lemma}[Removal of redundant segments: ordering]
  \label{lemma:removal-ordering}
  Let $\premd=\langle\tsec,\runs\rangle$ be a pre-model for a formula $\phi$
  with a redundant segment $\premd_{[j+1...k]}$. Let $\run\in\runs$ and let
  $\run'$ be the corresponding run in $\premd_{[j+1...k]}$. Then, $\run'\prec
  \run$.
\end{lemma}

\begin{proof}[Proof sketch]
  This is a direct consequence of the definition of the $\prec$ ordering. When a
  redundant segment is removed, all the $\next$-eventualities fulfilled in the
  removed segment do not change any distance vector, while those fulfilled later
  decrese their distance from the requesting points. For details, refer to Lemma
  6 in \cite{GeattiGMR21}.
\end{proof}

\begin{theorem}[Completeness]
 Let $\phi \in \langsltl$ and $T$ a complete tableau for $\Phi$. If $\Phi$ is satisfiable, then $T$ has a successful branch.
\end{theorem}

\begin{proof}
    Let $\model$ be a model for $\Phi$. As already noted, it is straightforward to build a pre-model for $\Phi$ from $\model$. Then, given a pre-model for $\Phi$, \Cref{lemma:greedy-pre-model-exists} ensures that a greedy pre-model for $\Phi$ exists. We can thus consider $\premd=\langle \tsec,\runs \rangle$ to be a greedy pre-model for $\Phi$. Now, given a complete tableau $T$ for $\Phi$, thanks to \Cref{lemma:branch-extraction} we can obtain a branch from $T$, with a sequence of step nodes $\branchvdef$ such that for each $a \in \tsec(k)$, there is some $\cs\in\nodei{v_k}$ such that $\csenc{\csatom}=\csenc{a}$ and $\csfor{\csatom}\subseteq\csfor{a}$ with $\csatomf{\cs}=\csatom$ for all $0 \leq k \leq m$. As already noted, we know that if $v_m$ is crossed, then it has to have been crossed by the $\pruner$ rule. If this was the case, however, it would mean there are other two-step nodes $v_i$ and $v_j$ with $i<j<m$ and $\nodei{v_i}=\nodei{v_j}=\nodei{v_m}$, and such that all the $\next$-eventualities requested in the three nodes and fulfilled between $v_{j+1}$ and $v_m$ are fulfilled between $v_{i+1}$ and $v_j$ as well. Since for each $a \in \tsec(k)$ we have one $\cs\in\nodei{v_k}$ such that $a=\csatomf{\cs}$ for all $0 \leq k \leq m$, this fact reflects onto the pre-model, hence $\tsec({i})=\tsec({j})=\tsec({k})$, and all the $\next$-eventualities requested in these atoms and fulfilled in $\premd_{[j+1...m]}$ are fulfilled in $\premd_{[i+1...j]}$ as well. That is, $\premd_{[j+1...m]}$ is a redundant segment, which contradicts the assumption that $\premd$ is greedy by \Cref{lemma:removal-correctness} and \ref{lemma:removal-ordering}.
\end{proof}

\section{Computational Complexity}\label{sec:complexity} 
\label{sec:complexity}


We now employ our tableau calculus to show that the
satisfiability for $\sltl$ is $\pspace$-complete. As $\sltl$ is a direct extension
of $\ltl$, it immediately follows that $\sltl$ is $\pspace$-hard. Hence, we need only show that the satisfiability of $\sltl$ formulae can be decided within $\pspace$. 

If our tableau-based decision procedure were to explicitly construct tableaux, then the resulting procedure would fail to be $\pspace$ as the tableau branches could be exponentially long (but not longer, as we will see) and the entirety of each branch would need to be kept in memory in order to check the $\loopr$ and $\pruner$ rules. However, such limitations can be overcome by incorporating \emph{non-determinism} into our tableau-based decision algorithm. Therefore, we let $n$ be the size of $\Phi$, and we confirm that any branch of a tableau is at most exponentially long in $n$.
\begin{lemma}
  \label{lemma:complexity:bound}
  Given a set of formulae $\Phi \subseteq \langsltl$, the length of a branch of the tableau for
  $\Phi$ is at most \emph{exponential} in $|\Phi|$.
\end{lemma}
\begin{proof} First, note that the construction of a branch stops whenever the $\loopr$ or $\pruner$ rule
  is triggered. The following two facts imply that the number of possible labels of a
  node is exponential: (1) The number of formulas that can appear in a constraint set is polynomial, (2) The number of different constraint sets is polynomial as there are
    no constraint sets $\setid{\cs}$ and $\setid{\cs'}$ in the same label such
    that $\setid{\cs}\subseteq\setid{\cs'}$.

  Hence, in any branch, after at most an exponential number of nodes, two nodes
  $u_i$ and $u_j$ with $\Gamma(u_i)=\Gamma(u_j)$ must appear. If this pair
  satisfies the conditions of the $\loopr$ rule, we are done (and the branch is
  exponential length). Otherwise, the branch construction continues. If the
  $\loopr$ rule is never triggered, the $\pruner$ will. Now, suppose by
  contradiction that the $\pruner$ rule is triggered only $2^{\Omega(p(n))}$
  nodes later with $p(n)$ a polynomial. Then, an exponential number of
  positions $j_0<j_1\ldots<j_k$, with $k\in\O(2^n)$, have to exist such that
  $\Gamma(u_{j_i})=\Gamma(u_{j_0})$ for all $0\le i\le k$. But then, observe
  that the set of $\next$-eventualities fulfilled between each $j_i$ and $j_{i+1}$ can
  only grow. This means that after a polynomial number of repetitions of
  $\Gamma(u_{j_0})$, the $\pruner$ rule must be triggered. Note that since
  $j_{i+1}-j_i\in\O(2^{c(n)})$, this contradicts the assumption.
\end{proof}

By means of the above Lemma, we can exploit and employ non-determinism in our tableau-based decision procedure to determine the complexity of $\sltl$.
\begin{theorem}
  Satisfiability of $\sltl$ is $\pspace$-complete.
\end{theorem}
\begin{proof}
  We show the existence of a deterministic procedure to decide the
  satisfiability of a set of $\sltl$ formulae $\Phi$, by providing a
  \emph{non-deterministic} one, and then appealing to the classic theorem by
  Savitch~\cite{Papadimitriou94} to state the existence of a deterministic one.

  Our procedure traverses a tableau non-deterministically by \emph{guessing} at
  each step which child of the current node to visit among the many created
  following the tableau rules. If a node accepted by the $\loopr$ rule is found,
  the procedure returns \emph{yes}. If a number of steps that exceeds the
  exponential upper bound established in \cref{lemma:complexity:bound} is
  reached, the procedure returns \emph{no} (i.e., the computation branch is
  rejected). To be able to do this in polynomial space, the procedure also
  guesses at each step a Boolean flag $\mathit{loop}$ that says whether the
  current node will be the one checked by the $\loopr$ rule to find a loop, i.e.,
  if the current position will be the one with the same label as the leaf. If
  $\mathit{loop}$ is guessed to be true, the current label is saved and kept in
  memory for later. Also, from now on, the set of eventualities fulfilled is
  kept in memory. When any node with the same label as the saved one is found,
  it means the $\loopr$ rule triggers on that branch, and the procedure can
  return \emph{yes}. Otherwise, the procedure proceeds. Note that the procedure illustrated here keeps the following polynomial amount of data in memory: (1) The current label, (2) The saved loop label, if any, and (3) The set of $\next$-eventualities fulfilled since the saved loop point, if    any.  Hence, the procedure works in polynomial space, showing $\sltl$ satisfiability is $\pspace$.
\end{proof}

\section{Concluding Remarks} 
\label{sec:conclusions}

We have introduced \emph{standpoint linear temporal logic} ($\sltl$)---a logic
fusing the multi-perspective reasoning offered by standpoint logic with the
dynamic reasoning offered by $\ltl$. As discussed, $\sltl$ permits one to model
diverse and potentially conflicting semantic commitments held by a set of
agents, while expressing how such commitments relate to temporal notions or
change throughout time. To automate $\sltl$-reasoning, we define a sound,
complete, and terminating tableau calculus, which supports (counter-)model
extraction witnessing the (non-)satisfaction of $\sltl$ formulae. We employed
our tableau calculus in an analysis of $\sltl$'s complexity, finding that the
incorporation of standpoint modalities---despite increasing the modelling
capacity of $\ltl$---preserves the $\pspace$-completeness of $\ltl$. 

Our contribution has practical implications. Not only it formally confirms that
standpoint and temporal reasoning can be combined at no cost in 
complexity, but the employment of a tree-shaped tableau \ala Reynolds paves the
way for its efficient implementation. Indeed, we plan to transfer the \emph{symbolic tableau}
technique applied in the BLACK satisfiability checker~\cite{GeattiGM19} to our tableau, thus providing an efficient SAT-based procedure to reason
in standpoint linear temporal logic.

\section*{Acknowledgments} Tim S. Lyon has received funding from the European
Research Council (Grant Agreement no. 771779, DeciGUT). Nicola Gigante
acknowledges the support of the PURPLE project, in the context of the AIPlan4EU
project's First Open Call for Innovators.

\bibliographystyle{kr}
\bibliography{bibliography}

\onecolumn 

\appendix
\section{Proofs}\label{app:proofs}


\begin{proof} ($\Rightarrow$)
Let $\premd=\langle\tsec,\runs\rangle$ be a pre-model of $\Phi$. For each $\run \in \runs$, we define a state sequence $\stseq(\run) = \tuple{\state_{0}, \state_{1}, \ldots}$ as follows: $p \in \state_{i}$ \iffi $p \in \run(i)$. We now define a temporal standpoint structure $\model = \tuple{\Pi,\lambda}$ such that 
\begin{enumerate}
    \item $\stseq(\run) \in \Pi$ \iffi $\run\in \runs$, and
    \item $\stseq(\run) \in \lambda(\st)$ \iffi $\st \in \csenc{\run(0)}$, for each $\st \in \spset$ and $\run\in \runs$
\end{enumerate}

We prove by induction on $\Phi$ that for any $\run\in \runs$, if $\psi \in \run(i)$, then $\model,\stseq(\run),i \models \psi$ for any $\psi \in \clo{\Phi}$ and any $i \geq 0$. 
Let $\run$ be an arbitrary run in $\runs$ and that $\psi \in \run(i)$ with $\psi \in \clo{\Phi}$.

 We note that the base case is immediate and follows directly from the definition of $\run$ and atom. We, therefore, consider the inductive step and we show the most interesting cases:

\begin{enumerate}

\item If $\standd{\st} \psi \in \csfor{a}$ for $a=\run(i)$, then there is $a'\in\tsec(i)$ such that $\psi \in \csfor{a'}$ and $\csenc{a'}\preceq {\st}$ by T1 in \Cref{def:timestamp}. And by the definition of a pre-model \Cref{def:pre-model} there exists a $\run' \in \runs$ such that $\run'(i)=a'$. By IH, we have that $\model,\stseq(r'), i \models \psi$, which implies that $\model,\stseq(r ), i \models \standd{\st'} \psi$.

\item Let $\standb{\st} \psi \in \csfor{a}$ for $a=r(i)$, and let $\run'\in \runs$ be an arbitrary run in $\run'\in\lambda(\st')$. Then, by R1 in \Cref{def:run}, we know that $\run'(i)=a'$ with $\csenc{a'}\preceq\stenc{\st}$. Moreover, we have $\psi \in a'(i)$ by T2 of \Cref{def:timestamp}. By IH, it follows that $\model,\stseq(r'), i \models \psi$, which implies that $\model, \stseq(r), i \models \standb{\st} \psi$ as the state sequence $\stseq(r')$ was chosen arbitrarily.

\item Let $\next \psi \in \csfor{a}$ for $a=r(i)$. Then, by R2 in \Cref{def:run} we have $\psi \in \csfor{a'}$ for $a'=r(i+1)$.
By IH, it follows that $\model,\stseq(r), i+1 \models \psi$, which implies that $\model, \stseq(r), i \models \standb{\st} \psi$.
\end{enumerate}

The rest of the cases are argued in a similar way.

\medskip

($\Leftarrow$) 
Let $\model = \tuple{\Pi,\lambda}$ be a model of $\phi$ with $\phi\in\stseq_0(0)$, and let us define a structure $\premd' = \tuple{\tsec',\runs'}$. We let $\tsec=\langle\tstamp'_0,\tstamp'_1,\dots\rangle$ and $\runs'$ be such that for each state sequence $\stseq \in \Pi$
\begin{enumerate}[label={(\alph*')}]
    \item $a'_{\stseq}\in\tstamp'_i$ such that $\csenc{a'_{\stseq}}=\{\st \mid \st\in\lambda^{-1}(\pi)\}$. 
    \item $a'_0\in\tstamp'_i$ such that $\csenc{a'_{\stseq}}=\stenc{*}$.
    \item $\run'_0\in\runs'$ such that $\run'_0(i)=a'_0\in\tstamp'_i$ and  $\csfor{a'_0}=\Delta'$ where $\Delta'$ is the \emph{minimal atom} (P2) obtained from $\Delta$ with $\psi\in\Delta$ if $\psi\in \clo{\Phi}$ and $\model, \stseq, i \models \phi$ and with $\Phi\subseteq\Delta'$ if $i=0$.

    \item $\run'_{\stseq}\in\runs'$ such that $\run'_{\stseq}(i)=a'_{\stseq}\in\tstamp'_i$ and $\csfor{a'_{\stseq}}=\Delta'$ where $\Delta'$ is the \emph{minimal atom} (P2) obtained from $\Delta$ with $\psi\in\Delta$ if  $\psi\in \clo{\Phi}$ and $\model, \stseq, i \models \phi$.

\end{enumerate}

We now define a structure $\premd = \tuple{\tsec,\runs}$ such that $\runs=\nest{\run\mid\run'\in\runs'}\cup\nest{\run_{\stence}\mid \stence\in E}$ and for each $i$ we have
\begin{enumerate}[label={(\alph*)}]
    \item for each $\stence\in E$,  $a_{\stence} \in \tstamp_i$ such that $\csenc{a_{\stence}}=\stence$, $\csfor{a_{\stence}}=\bigcap\{\csfor{a'}\mid a'\in\tstamp'_i, \csenc{a}\preceq\stence\}$ and $\run_{\stence}(i)=a_{\stence}$
    \item $a \in \tstamp$ if $a\in \tstamp'$ and there is no $a'\in\tstamp\setminus\nest{a_{\stence}\mid \stence\in E}$ with $\csenc{a}=\csenc{a'}$ and $\csfor{a'}=\csfor{a}$
    \item for each $\run'\in\runs'$ with $\run'(i)=a$, if $a\in\tstamp_i$ then let $\run(i)=a$. Otherwise, let $\run(i)=a'$ with $a'\in\tstamp_i$, $\csenc{a}=\csenc{a'}$ and $\csfor{a}=\csfor{a'}$.
\end{enumerate}

\noindent Let us show that $\premd= \tuple{\tsec,\runs}$ is a pre-model for $\phi$. We first show that $\tstamp_i$ is a timestamp for each $i\in\mathbb{N}$ by proving the 3 conditions in \Cref{def:timestamp}
Let $\tstamp$ be a timestamp in $\tsec$. Then,
\begin{enumerate}[label={T\arabic*}]
    \item For each $a \in \tstamp$, if $\standde \psi \in \csfor{a}$, then by construction there exists some $\stseq\in\Pi$ such that $\model, \stseq, i \models \standde\psi$. Then, by the semantics there is some $\stseq'\in\lambda(\ste)$ such that $\model, \stseq', i \models \psi$. Thus, $a'_{\stseq}\in\tstamp'$ with $\psi\in a'_{\stseq}$. If $a'_{\stseq}\in\tstamp$ then the condition is satisfied. Else, there is some $a\in\tstamp$ with $\csenc{a}=\csenc{a'_{\stseq}}$ and $\csfor{a'_{\stseq}}\subseteq\csfor{a}$, hence the condition is also satisfied.
    \item For each $a \in \tstamp$, if $\standbe \psi \in \csfor{a}$, then by construction there exists some $\stseq\in\Pi$ such that $\model, \stseq, i \models \standbe\psi$. Then, by the semantics, for any $\stseq'\in\lambda(\ste)$ we have $\model, \stseq', i \models \psi$. Thus, $a'_{\stseq}\in\tstamp'$ is such that $\psi\in a'_{\stseq}$, and hence if $a'_{\stseq}\in\tstamp$ then it satisfies the condition. It only remains to show that the newly introduced indexed atoms $a_\stence$ with $\stence\preceq\stenc{\ste}$ also satisfy the condition. This is clear by construction since $\csfor{a_{\stence}}=\bigcap\{\csfor{a'}\mid a'\in\tstamp'_i, \csenc{a}=\stence\}$ and the condition holds for every $a'\in\tstamp'$.
    \item Clearly follows by construction from (a)
\end{enumerate}

Let us now show that each $\run\in\runs$ is a run by proving the 3 conditions in \Cref{def:run}
\begin{enumerate}[label={R\arabic*}]
    \item First, notice that in $\premd'$, $\run'$ maps to an atom extracted from some $\stseq$ for all $i>0$, hence clearly for all $a'_i=\run'(i)$ and $a'_j=\run'(j)$, $\csenc{a'_i}=\csenc{a'_j}$. Then, by construction, $\run(i)$ maps either to $a'_i$ if $a'_i\in\tstamp_i$ or to some $a_i\in\tstamp_i$ such that $\csenc{a'_i}=\csenc{a_i}$, and analogously for $\run(j)$, hence the condition is satisfied.
    
    \item If $\next \psi \in \csfor{a_i}$ for $a_i=\run_{\stseq}(i)$, then for $a'_i=\run'_{\stseq}(i)$ we have $\csfor{a'_i}=\csfor{a_i}$. Hence $\model, \stseq, i \models \next\psi$ and thus $\model, \stseq, i+1 \models \psi$. Hence, $\psi\in\csfor{a'_{i+1}}$ for $a'_{i+1}=\run'_{\stseq}(i+1)$ and consequently $\psi \in \csfor{a_{i+1}}$ for $a_{i+1}=\run(i+1)$; 

    \item This is argued analogously to R2.
\end{enumerate}

Finally, having shown that each $\tstamp\in\tsec$ is a timestamp and each $\run\in\runs$ is a run, it only remains to show that the conditions for pre-modelhood are satisfied. First, we notice that, by construction every, timestamp $\tstamp_i$ and every atom $a\in\tstamp_i$ there is a run such that $\run(i)=a$. Then, (P1) from (b') and (d'), we have an atom $a\in \tstamp_0$ such that $\csenc{a}=\stenc{*}$ and $\Phi \subseteq \csfor{a}$. Finally, (P2)  follows by construction from (c'), (d') which make atoms minimal and (b), which removes the redundant atoms. Hence, $\premd$ is a pre-model for $\Phi$.
\end{proof}

\end{document}